\def\year{2020}\relax
\documentclass[letterpaper]{article} 
\usepackage{aaai20}  
\usepackage{times}  
\usepackage{helvet} 
\usepackage{courier}  
\usepackage[hyphens]{url}  
\usepackage{graphicx} 
\usepackage{graphicx}  
\usepackage{amsmath,amssymb,amsfonts}
\usepackage{lipsum}
\usepackage{booktabs}
\usepackage{algorithm}
\usepackage[noend]{algpseudocode}
\usepackage{multirow}
\usepackage{paralist}
\usepackage{xspace}
\usepackage{textcomp}
\usepackage{fancyhdr}

\usepackage{subfig}

\newcommand{\G}{\mathcal{G}}
\newcommand{\strips}{\textsc{strips}\xspace}

\usepackage{todonotes} 

\newcommand{\citet}[1]{\citeauthor{#1}~\shortcite{#1}}
\newcommand{\citep}{\cite}

\usepackage{amsthm}
\theoremstyle{definition}
\newtheorem{definition}{Definition}
\newtheorem{theorem}{Theorem}
\newtheorem{lemma}{Lemma}[theorem]
\newenvironment{proof_sketch}{\proof}{\endproof}
\usepackage{tabularx}

\usepackage{enumitem}

\title{Partial-Order, Partially-Seen Observations of Fluents or Actions \\ for Plan Recognition as Planning}

\author{%
	Jennifer M. Nelson\textsuperscript{\rm 1} \and Rogelio E. Cardona-Rivera\textsuperscript{\rm 1, 2}\\ 
	\textsuperscript{\rm 1}School of Computing, \textsuperscript{\rm 2}Entertainment Arts and Engineering Program\\ 
University of Utah\\
Salt Lake City, UT 84112 USA\\
jennifer.m.nelson@utah.edu, rogelio@cs.utah.edu 
}
\begin{document}

\maketitle

\begin{abstract}
This work aims to make plan recognition as planning more ready for real-world scenarios by adapting previous compilations to work with partial-order, half-seen observations of both fluents and actions. We first redefine what observations can be and what it means to satisfy each kind. We then provide a compilation from plan recognition problem to classical planning problem, similar to original work by \citeauthor{ramirezGeffner09}, but accommodating these more complex observation types. This compilation can be adapted towards other planning-based plan recognition techniques. Lastly we evaluate this method against an ``ignore complexity" strategy that uses the original method by \citeauthor{ramirezGeffner09}.
Our experimental results suggest that, while slower, our method is equally or more accurate than baseline methods; our technique sometimes significantly reduces the size of the solution to the plan recognition problem, \textit{i.e,} the size of the optimal goal set. We discuss these findings in the context of plan recognition problem difficulty and present an avenue for future work.

\end{abstract}
%

\section{Introduction}

\noindent Plan recognition is the problem of identifying the plans and goals of an agent, given some observations of their behavior(s)~\cite{sukthankar2014plan}. Plan recognition has applications wherever it's useful for a system to anticipate an agent's actions or desires. This variety of applications includes robot-human coordination~\cite{talamadupula2014coordination}, human-computer collaboration~\cite{lesh1999using}, assisted cognition~\cite{pentney2006sensor}, network monitoring~\cite{sohrabi2013hypothesis}, interactive narratives~\cite{cardona2015symbolic}, and language recognition~\cite{carberry1990plan,zukerman2001natural}. 

\citet{ramirezGeffner09} realized that plan recognition problems were very similar to classical planning problems, and created a formulation to compile recognition problems into planning problems ready for off-the-shelf planning algorithms. Previously, plan recognition relied on specialized algorithms and handcrafted libraries. Rather than rely on a library of possible plan-goal pairs, \citeauthor{ramirezGeffner09}'s formulation relies on a set of possible goals and a \textit{domain theory} describing possible actions. It assumes that any plan which reaches a possible goal at optimal cost but also ``explains'' all observations (in order) is part of the optimal solution set to a recognition problem.

In addition to defining an optimal solution set, \citet{ramirezGeffner09} also relaxed its own optimality assumption to allow suboptimal \textit{approximate} solutions computed with faster algorithms. This also allowed solutions to ``skip'' some observations if necessary. \citet{ramirez2010probabilistic} also relaxed the optimality assumption, such that goals whose optimal plans differed significantly from the observations were considered less likely. \citet{sohrabi2016plan} further relaxed the optimality assumption, admitting that observation sequences may be non-optimal, noisy, or missing segments. It assumed observations of single fluents.

The methods above all assume total-ordered fully specified observations, though real-world applications may be more complex. One might observe artifacts of past actions, but not know the order in which those artifacts appeared. One might see the actor pick something up, but not know if it were the key or the coin (a half-specified observation). One might later observe that the key is missing from that spot (a fluent observation). This is important information if the agent's goal is behind a locked door, but current methods cannot use it. Our work provides methods to utilize this information. In this paper we modify the original \cite{ramirezGeffner09} compilation, but our definitions for observations and our use of ordering fluents can be adapted for any of the methods mentioned above. We focus only on the ``optimal" set of answers for complex observations, leaving relaxations to future work.

\section{Motivating Example}
\begin{figure}
    \centering
    \includegraphics[width=\linewidth]{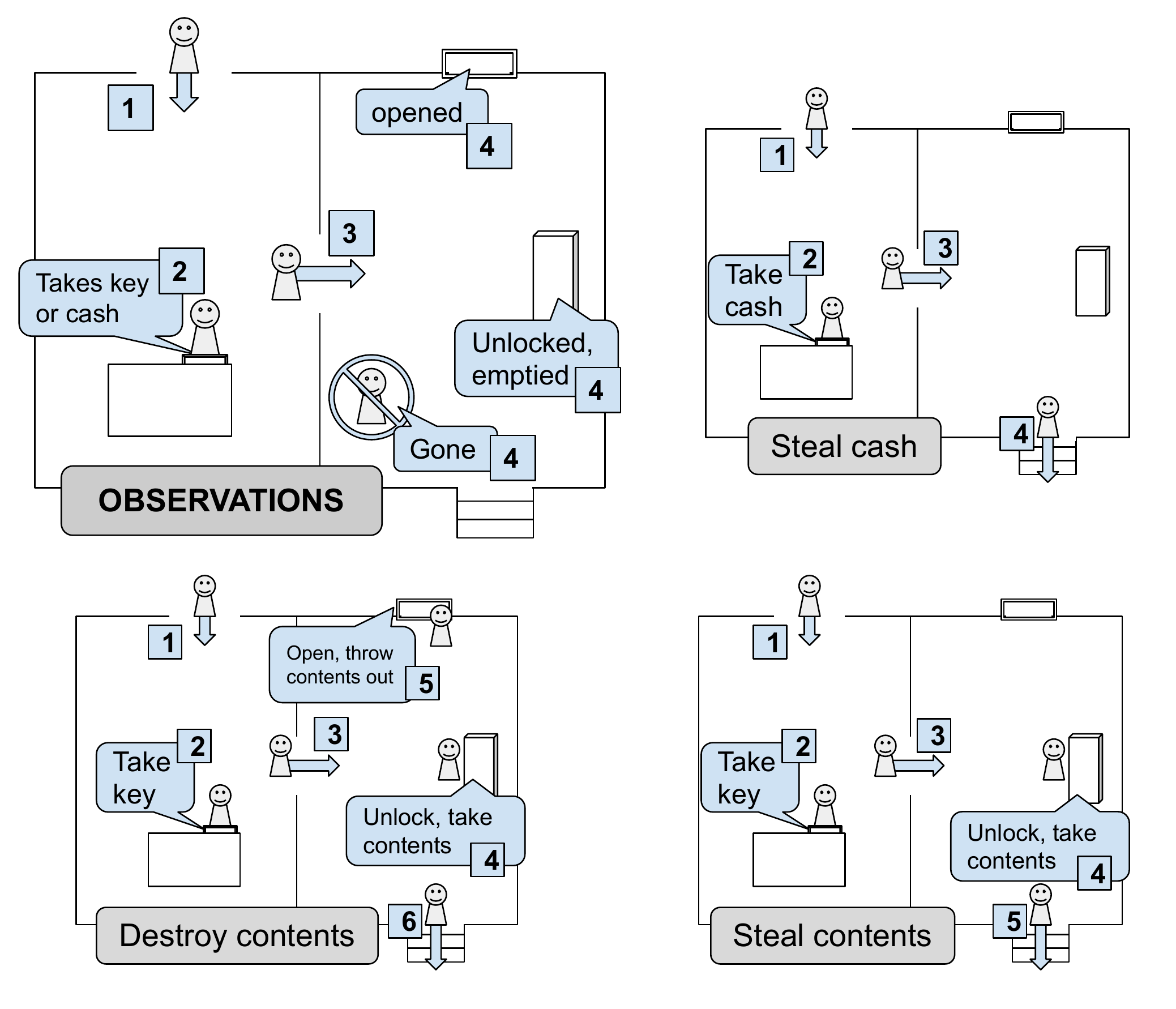}
    \caption{DetectiveBot's observations, and unconstrained plans for the possible motivations }
    \label{fig:MotivatingExample}
\end{figure}

We illustrate our method with the following scenario: DetectiveBot is trying to solve a breaking-and-entering case at a museum. Cameras record the culprit breaking into the museum office, rifling through the manager's top drawer, pocketing something, then sprinting into an unfilmed backroom. DetectiveBot inspects the backroom: it contains a single window (opened), a chest (unlocked and emptied), and a stairwell towards an exit. DetectiveBot wants to figure out the culprit's motives. Were they stealing cash from the managers drawer? Were they stealing the contents of the chest? Or were they destroying the contents of the chest?

DetectiveBot models this situation. It knows the culprit took either cash or a key to the chest from the drawer, then entered the back room. DetectiveBot does not know what order things happened in the backroom, but it knows that by the end the window was opened, the chest was unlocked and emptied, and the culprit was gone. First DetectiveBot computes plans for what the culprit would've done for each of their three possible motives, unconstrained by DetectiveBot's observations. (Figure \ref{fig:MotivatingExample}) Then it computes plans for each of the three possible motives, such that each plan also ``sees" the observations. DetectiveBot compares the unconstrained plans to their constrained counterparts, and discovers that only one pair has identical costs: destroying the contents.

\section{The Goal Recognition Problem with Complex Observation Constraints}

\paragraph{Planning Background} In this paper, we rely on the formulation of plan recognition as \textbf{classical planning}. Classical planning is a model of problem solving, wherein agent actions are fully observable and deterministic. Classical problems are typically represented in the \strips formalism~\cite{fikes1971strips}; a \strips planning problem is a tuple $P = \left\langle F, I, A, G, f_{\textnormal{cost}} \right\rangle$ where $F$ is the set of fluents, $I \subseteq F$ is the initial state, $G \subseteq F$ is the set of goal conditions, and $A$ is a set of actions. Each action is a triple $a = \left\langle \textsc{pre}(a), \textsc{add}(a), \textsc{del}(a) \right\rangle$, that represents the precondition, add, and delete lists respectively, all subsets of $F$. A state is a set of conjuncted fluents, and an action $a$ is applicable in a state $s$ if $\textsc{pre}(a) \subseteq s$; applying said applicable action in the state results in a new state $s' = (s\textnormal{\textbackslash} \textsc{del}(a)) \cup \textsc{add}(a)$ and incurs a non-negative cost determined by the function $f_{\textnormal{cost}}: A \xrightarrow{} R^{0+}$.

The solution to a planning problem $P$ is a plan $\pi = [a_1,...,a_m]$, a sequence of actions $a_i \in A$ that transforms the problem's initial state $I$ to a state $s_m$ that satisfies the goal; \textit{i.e.} $G \subseteq s_m$. The cost c$(\pi)$ of a plan $\pi$ is $\Sigma f_{\textnormal{cost}}(a_i)$ for all actions $a_i \in \pi$. A plan segment is a segment of a plan, denoted $\pi^{k}_{j} = [a_j,...,a_k] \ (a_i \in A, 1 \leq j \leq k\leq m)$. 

The execution trace ${trace(\pi,I)=[I,a_1,s_1,...,a_m,s_m]}$ of plan $\pi$ from initial state $I$ is defined as the alternating sequence of states and actions, starting with $I$, such that $s_i$ results from applying $a_i$ to state $s_{i-1}$.

\paragraph{Handling Complex Observations} Our formulation 
is based on the formulation 
by \citeauthor{ramirezGeffner09}, but we relax the assumption that all observations are totally ordered and grounded actions. Instead, we
allow observations to be either an observed action or a set of observed fluents. Further, we allow partial orderings in the observations as well as partially-specified observations via sets of possible observations. 

Fundamental to this formulation are \textbf{observation groups}, which impose constraints on the observations they contain. We describe two types of observations, three types of groups, and what it means for a plan to \textbf{satisfy} each. Because groups can nest within each other, we describe satisfaction of an outer group
in terms of the satisfaction of its nested member(s) by a plan \textit{through} a plan segment. Because a member might be a simple observation, we also describe the satisfaction of a simple observations in terms of a plan \textit{through} a plan segment.

\begin{definition}
    An \textbf{action observation} $o$ paired with action $a \in A$ is satisfied by the plan $\pi$ through segment $\pi_j^k$ iff $a = a_i$ for some $a_i \in \pi_j^k$.
\end{definition}

A fluent observation is a set of fluents, and is satisfied by plan segments that mark out a time period where those fluents are true for some state. This definition relies on initial state $I$, which is later set in the definition of a Plan Recognition Problem.

\begin{definition}
    A \textbf{fluent observation} $o$ paired with fluents $F_o=\{f_1,...,f_m\} (f_i\in F)$ is satisfied by the plan $\pi$ with initial state $I$ through segment $\pi_j^k$ iff $F_o \subseteq s_i$ for some $s_i \ (j \leq i \leq k)$ in $trace(\pi,I)$.
\end{definition}

Note that the actions in the plan segment do not need to contribute to the observed fluents for this notion of satisfaction. The plan segment merely marks a time period in which the fluents were observed. It may be that $F_o$ was true since the initial state, but was not observable until much later. Our intent is to rule out goal-plan pairs where the plan never co-occurs with the fluents in $F_o$ being true. 

Now we define \textbf{ordered groups}, who impose ordering constraints on members. A member can be either another group or a single observation. An ordered group can only be satisfied by a plan segment if that segment can be split into chunks that satisfy each member \emph{in order}. (These chunks are the reason we define satisfaction in terms of plan segments.)

\begin{definition}
    An \textbf{ordered observation group} $\Theta_{<}=[\theta_1,...,\theta_n]$ is a totally ordered sequence of observation groups and/or single observations. A plan $\pi$ satisfies $\Theta_{<}$ through the plan segment $\pi_j^k$ iff there exists a monotonically increasing function of the form $f: [1,n+1] \rightarrow [j,k+1]$, $f(n+1) = k+1$, which maps members of $\Theta_{<}$ to segments of $\pi_j^k$ such that $\pi_{f(i)}^{f(i+1)-1}$=$[a_{f(i)},...,a_{f(i+1)-1}]$ satisfies $\theta_i$. 
\end{definition}

The function $f$ above is used to ensure ordering. It maps subsequent group members to subsequent plan segments. $f(i)$ marks the beginning of $\theta_i$'s plan segment. We allow no gaps in plan segments, so $\theta_i$'s segment ends right before $\theta_{i+1}$'s segment begins, and $\theta_n$'s segment ends where the whole plan segment ends.

Next we define \textbf{unordered groups}, who impose no constraints on members, but are only satisfied when all members are. When embedded in an ordered group, these form the \textit{partial} part of \textit{partial order}.

\begin{definition}
    An \textbf{unordered group} $\Theta_{\wedge}=\{\theta_1,...,\theta_n\}$ is a set of observation groups and/or single observations that have no ordering constraints with respect to each other. A plan $\pi$ satisfies $\Theta_{\wedge}$ through the segment $\pi_j^k$ iff $\pi_j^k$ satisfies all members.
\end{definition}

Lastly, we define \textbf{option groups}. Unlike the other groups, this group may contain only single observations, not nested groups, and is intended to describe a set of mutually exclusive \textit{possible} observations. This is how we support partially-seen observations: by transforming it into an option group of all its possible interpretations. This group is satisfied if at least one of its members is satisfied.

\begin{definition}
    An \textbf{option group} $\Theta_{\oplus}=|o_a,...,o_b|$ is a set of single observations where it is uncertain which of them was the true observation. A plan $\pi$ satisfies $\Theta_{\oplus}$ through the segment $\pi_j^k$ satisfies at least one member. A plan which satisfies more than one member is not considered more likely than a plan which satisfies only one member.
\end{definition}

With the above definitions, we mark out a modified version of the plan recognition problem. This is largely the same as previous work, but replaces a total-order constraint on observations with partial-order constraints and option groups.
\begin{definition}
    A \textbf{plan recognition problem over a domain theory} is the tuple $T=\langle P,\G,\Theta \rangle$ where $P = \langle F,I,A \rangle$ is a planning domain, $\G$ is the set of possible goals $G, G \subseteq F$, and $\Theta$ is an observation group as defined above. A solution to $T$ is one of the goals $G \in \G$ which has an optimal plan $\pi= [a_1,...,a_n] \ (a_i \in A)$ that also satisfies $\Theta$.
\end{definition}

\section{Compilation to Planning Problem}
We compile a goal recognition problem into a planning problem $P'[G']$ such that a solution to $P'[G']$ ``explains'' the observations nested in $\Theta$, while respecting $\Theta$'s ordering constraints and not double-explaining observations in an option group. If an optimal solution to $P'[G']$
has the same cost as an optimal solution to $P[G]$, $G$ and the plan solving $P'[G']$ are considered a solution to the plan recognition problem.

To ensure a solution to $P'[G']$ respects $\Theta$'s constraints, we use ordering fluents to ensure an explanation may only happen \textit{after} its predecessors, and that only one explanation is allowed per observation, or per option group. Let $nest(\theta)$ denote the set of all observations nested within $\theta$ or its subgroups.

\begin{definition}
    For the goal recognition problem \mbox{$T=\langle P,\G,\Theta \rangle$} where $P = \langle F,I,A \rangle$, the transformed planning problem for each $G \in \G$ is defined as \mbox{$P'[G'] = \langle F', I', A', G' \rangle$} such that:
    \begin{itemize}
        \item $F' = F \cup F_e$, where $F_e = \{p_{o_i} | \forall o_i \in nest(\Theta)\}$
        \item $I' = I$
        \item $A' = A \cup A_e$, where $A_e = \{e_{o_i} | \forall o_i \in nest(\Theta)\}$, and 
        \item $G' = G \cup F_e$.
    \end{itemize}
\end{definition}
\noindent We further define $A_{e}$ and $F_{e}$, and later show that a solution to $P'[G']$ satisfies $\Theta$.

\begin{definition}
    The explanation action $e_{o_i}$ for the fluent observation $o_i$ corresponding to fluents $F_{o_i} \subseteq F$ is a dummy action that marks that $F_{o_i}$ is observed, defined as:
    \begin{itemize}
        \item $\textsc{pre}(e_{o_i}) = F_{o_i} \cup \{\neg p_{o_i}\} \cup \{p_{o_{pre}} | o_{pre} \in B\}$ where $B$ is the set of all observations nested in any group immediately preceding a group that $o_i$ is nested within.
        \item $\textsc{add}(e_{o_i}) = \{p_{o_i}\}$
        \item $\textsc{del}(e_{o_i}) = \emptyset$
        \item $f_\textnormal{cost}(e_{o_i}) = 0$
        \item $p_{o_i} = p_{o_j}$ for all $o_j $ in the same option group as $o_i$ 
    \end{itemize}
\end{definition}
\noindent This definition is based off those of \citet{sohrabi2016plan}, except multiple fluents can be included in the same observation. A metric planner is needed to work with this zero-cost action, or the cost of these actions can be subtracted post-planning. 
\begin{definition}
    The explanation action $e_{o_i}$ for the action observation $o_i$ corresponding to action $a \in A$ is an action identical to $a$ but with additional ordering fluents:
    \begin{itemize}
        \item $\textsc{pre}(e_{o_i}) = \textsc{pre}(a) \cup \{\neg p_{o_i}\} \cup \{p_{o_{pre}} | o_{pre} \in B\}$ where $B$ is the set of all observations nested in any group immediately preceding a group that $o_i$ is nested within.
        \item $\textsc{add}(e_{o_i}) = \textsc{add}(a) \cup \{p_{o_i}\}$
        \item $\textsc{del}(e_{o_i}) = \textsc{del}(a)$
        \item $f_\textnormal{cost}(e_{o_i}) = f_\textnormal{cost}(a)$
        \item $p_{o_i} = p_{o_j}$ for all $o_j $ in the same option group as $o_i$ 
    \end{itemize}
\end{definition}
\noindent Note that explanation actions have the precondition $\neg p_{o_i}$, but add $p_{o_i}$ as an effect. As no action removes $p_{o_i}$, this means an explanation action cannot be used twice. Additionally, explaining an observation in an option group prevents all other explanations from that option group from being used. 

\begin{definition}\label{solution_def}
    The solution to $T = \langle P, \G, \Theta \rangle$ is the set \mbox{$\G^* = \{ G\in\G : \exists \pi$ satisfying $\Theta$ and optimally solving $P[G]$\} }
\end{definition}
\noindent In the next section, we prove that our compilation indicates members of $\G^*$: $G$ is in $\G^*$ when the optimal plan for $P'[G']$ costs the same as an optimal plan for $P[G]$. To find all members of $\G^*$, one would optimally solve $P'[G']$ and $P[G]$ for all $G$ in $\G$, and compare costs.

\section{Proofs}
In this section we present two main proofs. The first is a proof that our compilation indicates if a goal is in the solution to a goal recognition problem; \textit{i.e.} if the goal has an observation-satisfying plan that optimally reaches the goal. The second is a proof that solving the compiled problem will never yield an optimal goal set of size greater than the optimal goal set size achieved by solving the problem compiled as in \citeauthor{ramirezGeffner09} when ignoring complex observations. 

\subsection{Goal Recognition Problem is Solved}
We prove that our compilation produces a planning problem that solves the goal recognition problem in two steps. We first prove that any plan $\pi$ for the compiled problem $P'[G']$ has a corresponding plan $\psi(\pi)$ of equivalent cost that solves $P[G]$. We then prove that if $\pi$ is an optimal plan for $P'[G']$, then $\psi(\pi)$ satisfies $\Theta$ and (by first proof) solves $P[G]$ with the same cost as $\pi$. If this cost is the same as an optimal plan for just $P[G]$, then $G$ is a solution to $T$.
\begin{theorem}\label{corresponding_plan}
A plan $\pi$ for $P'[G']$ has a corresponding plan $\psi(\pi)$, solving $P[G]$, such that $c(\pi) = c(\psi(\pi))$. 
\end{theorem}
\begin{proof}
For $\pi$, let $\psi(\pi)$ be the same sequence of actions, but with fluent observation explanations removed, and action observation explanations replaced with their corresponding action in $A$. Because fluent explanations have no cost, and action explanations have cost identical to their corresponding action, $c(\pi)=c(\psi(\pi))$. Fluent explanations have no effect save for ordering fluents, and action explanations are identical to their corresponding actions, save for ordering fluents. Since $G$ does not include any ordering fluents, $\psi(\pi)$ still achieves $G$. 
\end{proof}

\begin{theorem}\label{satisfies_theta}
If plan segment $\pi_j^k = [a_j, ..., a_k] (a_i \in A')$ achieves all $p_{o_i}$ for $o_i \in nest(\Theta)$, then $\psi(\pi_j^k)$ satisfies $\Theta$.
\end{theorem}
\begin{proof}
We prove this theorem through a series of Lemmas showing that such a plan segment will satisfy every observation and observation group; each Lemma corresponds to a different complex observation type.
%
    \begin{lemma}[Individual Observations]
    \label{o_iLemma}
    If $\pi_j^k$ achieves $p_{o_i}$, 
    then $\psi(\pi_j^k)$ satisfies $o_i$.
    \end{lemma}

    \begin{proof}
    The only way for $\pi_j^k$ to achieve $p_{o_i}$ is through explanation action $e_{o_i}$. If $o_i$ is an observation of action $a$, then $e_{o_i}$ is translated to $a$ in $\psi(\pi_j^k)$, satisfying $o_i$. If $o_i$ is an observation of fluents $F_{o_i}$, then $e_{o_i}$ has $F_{o_i}$ as precondition, so $F_{o_i}$ must exist in the execution trace for $\pi_j^k$, and thus in the trace for $\psi(\pi_j^k)$. In either case, $\psi(\pi_j^k)$ satisfies $o_i$. 
    \end{proof}

    \begin{lemma}[Option Group]
    \label{optionGroupLemma}
    If $\pi_j^k$ achieves any $p_{o_i}$ for $o_i$ in the option group $\Theta_{\oplus}$, then $\psi(\pi_j^k)$ satisfies $\Theta_{\oplus}$.
    \end{lemma}

    \begin{proof}
    If $\pi_j^k$ achieves a particular $p_{o_i}$ for $o_i \in \Theta_{\oplus}$, then by Lemma \ref{o_iLemma}, $\psi(\pi_j^k)$ satisfies $o_i$. By satisfying a member of $\Theta_{\oplus}$, $\psi(\pi_j^k)$ satisfies $\Theta_{\oplus}$.
    \end{proof}

    \begin{lemma}[Unordered Group]
    \label{unorderedGroupLemma}
    If $\pi_j^k$ achieves all $p_{o_i}$ for $o_i \in nest(\Theta_{\wedge})$, then $\psi(\pi_j^k)$ satisfies $\Theta_{\wedge}$.
    \end{lemma}

    \begin{proof}
    $\psi(\pi_j^k)$ satisfies every simple observation contained directly in $\Theta_{\wedge}$, per Lemma \ref{o_iLemma}. $\psi(\pi_j^k)$ also satisfies any contained option groups, per Lemma \ref{optionGroupLemma}. If $\Theta_{\wedge}$ contains unordered groups, this is equivalent to containing the unordered group's members directly. Any contained ordered groups are also satisfied by $\psi(\pi_j^k)$, via Lemma \ref{orderedGroupLemma}. 
    \end{proof}

    \begin{lemma}[Ordered Group]
    \label{orderedGroupLemma}
    If $\pi_j^k$ achieves all $p_{o_i}$ for $o_i$ nested in the ordered group $\Theta_{<}=[\theta_1,...,\theta_n]$, $\psi(\pi_j^k)$ satisfies $\Theta_{<}$.
    \end{lemma}
    
    \begin{proof}
    Let ${f:[1,n+1]\rightarrow[j,k+1]}$ be a function
    where $f(n+1)=k+1$ and $f(i)$ is the index of the \textit{first} explanation action for any $o\in nest(\theta_i)$. Segment 
     $\pi_{f(i)}^{f(i+1)-1}$   
    then achieves all $p_o$ for $o\in nest(\theta_i)$, since the explanation action at $f(i+1)$ has \{$p_o$ for $o \in \theta_i$\} as a precondition. Via the other Lemmas, $\psi(\pi_{f(i)}^{f(i+1)-1})$ satisfies $\theta_i$.

    Let $f_\psi$ be a similar function of the form:
    \begin{equation*}
     f_\psi:[1,n+1]\rightarrow[1,|\psi(\pi_j^k)|+1]   
    \end{equation*}
    where $f_\psi(n+1) =|\psi(\pi_j^k)|+1$ and $f_\psi(i)$ maps to where the action at $f(i)$ \textit{would} be if the $\psi(\cdot)$ transformation did not remove/transform it. 
    This way, $f_\psi$ creates plan segments corresponding to the plan segments $f$ creates, such that
    
    \begin{equation*}
        \psi(\pi_{f(i)}^{f(i+1)-1}) = (\psi(\pi_j^k))_{f_\psi(i)}^{f_\psi(i+1)-1}        
    \end{equation*}

    Since (as mentioned above), the left-hand side of this equation satisfies $\theta_i$, so too does the right-hand side.
    This makes $f_\psi$ a non-monotonically increasing function which separates $\psi(\pi_j^k)$ into sections which satisfy each member of $\Theta_<$, and with it, $\psi(\pi_j^k)$ satisfies $\Theta_<$.
\end{proof}

Lemmas \ref{unorderedGroupLemma} and \ref{orderedGroupLemma} recurse into themselves if an unordered group contains an ordered group (or vice versa), but are satisfied by the base case where a group contains only simple observations or option groups. 
\\ 
With Lemmas \ref{o_iLemma} - \ref{orderedGroupLemma}, we prove a plan segment achieving all $p_{o_i}$ has a corresponding plan that satisfies $\Theta$.
\end{proof}
 An optimal solution to $P'[G']$ necessarily achieves all $p_{o_i}$, and so by theorem \ref{satisfies_theta}, has a corresponding plan that satisfies $\Theta$. With Lemma \ref{corresponding_plan}, we prove that this corresponding plan also solves $P[G]$. If the cost of this plan is the same as the cost for an optimal plan to just $P[G]$ (not constrained by $\Theta$), then \textbf{a plan exists that satisfies $\Theta$ \textit{and} optimally solves $P[G]$.} By definition \ref{solution_def}, $G$ is a solution to $T$.

\subsection{No Worse than Ignoring Complexity}
We begin by defining an ``ignore complexity" strategy for simplifying observation groups to a total-ordered, fully-specified form the compilation in \citet{ramirezGeffner09} can handle. This strategy removes fluent observations and option groups, reduces unordered groups to a single member, then simplifies zero- or no-member groups until just an ordered group is left. We choose this strategy over strategies that try different orderings/option group members because the other strategies would take exponentially longer to solve, requiring as many tries as there are orderings of unordered groups and combinations of option group choices. We sketch a proof that using complex observations will always be more accurate than or equally accurate to ignoring them. (Accuracy is measured by number of goals indicated: fewer false positives is more accurate). 
\begin{theorem}\label{no_worse}
Given $T_{cpx}=\langle P, \G, \Theta \rangle$ and $T_{ign}=\langle P, \G, \Theta_{ign} \rangle$, where $\Theta_{ign}$ removes some number of observations from $\Theta$ without altering ordering constraints, \mbox{$|\G^*_{cpx}| \leq |\G^*_{ign}|$}, where $\G^*_{cpx}$ is the solution set to $T_{cpx}$ and $\G^*_{ign}$ is the solution set to $T_{ign}$.
\end{theorem}
\begin{proof_sketch}
    Assume $|\G^*_{cpx}| > |\G^*_{ign}|$. Then there exists some $G \in \G$ such that $G \in \G^*_{cpx}$ but $G \not\in \G^*_{ign}$. This means an explanation action for some observation in $\Theta_{ign}$ created a larger cost for the optimal plan for $P'[G']$ compiled for $T_{ign}$, making $c^*(P'[G']) > c^*(P[G])$ and eliminating $G$ from $\G^*_{ign}$. Because the observations in $\Theta_{ign}$ are a subset of those in $\Theta_{cpx}$, that explanation action will also incur a cost for $P'[G']$ compiled for $T_{cpx}$, eliminating $G$ from $\G^*_{cpx}$. This contradicts the premise, so $|\G^*_{cpx}| \leq |\G^*_{ign}|$.
\end{proof_sketch}

\section{Experimental Evaluation}
We evaluate the proposed formulation against the ``ignore complexity'' strategy for accomodating complex obervations using only the formulation in \citet{ramirezGeffner09}. We use the same domains and plan recognition problems from that work but generate new \textit{complex} observations according to two parameters. The metric we're concerned with is the number of incorrect goals in the optimal goal set $\G^*$. By this metric we often perform better, and never perform worse. In some domains the ``ignore complexity" strategy often found no incorrect goals, leaving our formulation no room for improvement. We report how often this occurred, and focus on cases where we could improve. In general, our method is slower, taking longer to generate plans.

\subsection{Method}

\textbf{Hypotheses}\ \ We hypothesize that the size of our optimal goal set will often be smaller than the size of the optimal goal set computed using simplified observations. We also measure the time it takes to compute the optimal goal set. This will be domain dependent, but we expect to see a general trend favoring one method or the other.

\textbf{Apparatus}\ \ 
We developed our software\footnote{Available at 
\url{https://github.com/qed-lab/Complex-Observation-Compiler}
} by expanding the original plan recognition as planning code developed by \citet{ramirezGeffner09}.  Our software ran atop Centos 7.2 Linux with the 3.10 kernel, deployed on hardware equipped with an $3.60$GHz Intel Core i7-4790 Processor, $32$GB DDR3 $1600$MHz overclocked RAM, and $240$GB Intel 540 Solid State Drive. Optimal plans were generated using A* search with admissible h-max heuristic, but pruning paths when the estimated cost-to-goal reached a given threshold. When computing plans in compiled $P'[G']$, we set this threshold to the optimal cost of $P[G]$, which was pre-computed and not counted towards measures of time. This reduced the time it took to eliminate incorrect goals from being included in $\G^*$. For each $P'[G']$, we allotted ten times the time taken to compute $P[G]$, with a minimum of 20 seconds. (In the logistics domain, this was sometimes still not enough time.)

\textbf{Procedure} \ \ 
For every plan recognition problem in \citet{ramirezGeffner09}, consisting of a domain $P = \langle F,I,A \rangle$ and possible goals $\G$ of which $G_{true}$ is correct, we found an optimal plan for $P[G_{true}]$, computed the trace of that plan, and generated complex observations from both the plan (\textbf{A}) and trace (\textbf{A+F}). For each plan/trace, we generated 3 observation sets per setting, varying setting by amount of observations of unknown order (\textbf{U\%}) and amount of `debound' observations missing a parameter (\textbf{D\%}). To generate observations, we randomly removed half the optimal plan (or trace), and removed 90\% of fluents from states in each trace. With the remaining plan (or trace) we placed \textbf{U\%} of observations into unordered groups of size $\sim\!\!3$. (So, \textbf{U\%} = 100\% means every observation was placed in a small unordered group, not that all observations were completely unordered.) Lastly, we `debound' \textbf{D\%} of action observations by removing one parameter and replacing the observation with an option group of all matching observations. (I.e. (eat H)$\rightarrow$(eat ?)$\rightarrow$$\Theta_{\oplus}=|$(eat A),...,(eat Z))$|$) \textbf{D\%} was only applied to action observations with at least one parameter. Altogether, this forms the plan recognition problem $\langle P, \G, \Theta \rangle$.

We varied \textbf{U} and \textbf{D} over five settings: 
\begin{itemize}[noitemsep,topsep=0pt,parsep=0pt,partopsep=0pt]
    \item No Obscuration: \textbf{U=0\% D=0\%} 
    \item Vary Debind: \textbf{U=0\% D=25\%}
    \item Vary Unorderedness: \textbf{U=25\% D=0\%}
    \item Vary Unorderedness: \textbf{U=50\% D=0\%}
    \item Vary Both: \textbf{U=50\% D=25\%}
\end{itemize}

For each problem $\langle P, \G, \Theta \rangle$, and each $G \in \G$, we compiled two planning domains: $P'[G']$ (this work's compilation) and $P'_{ign}[G'_{ign}]$ (Using the ``ignore complexity" strategy and compiling as in \cite{ramirezGeffner09}). We compare the optimal costs of both to the optimal cost for $P[G]$ (precomputed). If $c^*(P'[G']) = c^*(P[G])$, we place $G$ in $\G^*_{cpx}$. If  $c^*(P'_{ign}[G'_{ign}])= c^*(P[G])$, we place $G$ in $\G^*_{ign}$.

\subsection{Analysis}

We conducted tests over four domains: Block-Words, Ipc-Grid, Grid-Navigation, and Logistics. For the latter three, most problems were perfectly solved with the ``ignore complexity" strategy, leaving no room for improvement. We removed instances where the ``ignore complexity'' strategy resulted in an empty observation set. (This occurred \textbf{38} times, mostly in the [\textbf{A+F} \textbf{U:50\%}\textbf{D:25\%}] setting.) Table \ref{tab:Data} reports the number of perfectly solved problems (Opt), problems not improved upon (Un), and problems improved upon (Imp), per setting, per domain. Table \ref{tab:Data} also reports the average number of observations per method when improvable ($|\Theta_{ign/cpx}|$ Opt) and when not improvable ($|\Theta_{ign/cpx}|$ Un/Imp), the average size of the solution set when improvable ($|\G^*_{ign/cpx}|$), and the average time to compute (whether or not improvable) (time$_{ign/cpx}$). Error rates indicate a 95\% confidence interval. When computing the observation set size, we considered an option group to have size 1.

We conducted an independent t-test, not assuming same variance, comparing the sizes of solution sets when improvable. It found a statistically significant difference between $|\G^*_{ign}|$($\mu$=3.91, $\sigma$=2.99) and $|\G^*_{cpx}|$($\mu$=2.51, $\sigma$=1.73) ($t$($df$=2372.97)=15.620, $p<$0.01, $\mu_{ign}$--$\mu_{cpx}$=1.40) T-tests also found statistically significant differences ($p<$0.01) for each domain, with Block-Words having the largest t-value ($t$($df$=1703.41)=14.15) and difference ($\mu_{ign}$--$\mu_{cpx}$=1.64). 

Notice that the [\textbf{A} \textbf{ U:0\% D:0\%}] setting produces identical solution set sizes. This is because without complex observations, our compilation is equivalent to that in \citet{ramirezGeffner09}. This changes for the \textbf{A+F} mode, where the ``ignore complexity'' strategy removes fluent observations.

Figures \ref{fig:G_Comparison} and \ref{fig:TimeComparison} show the results from Block-Words in more detail. Notches represent a 95\% confidence interval around the median value, and dashed lines represent mean. These are divided into action observations only (\textbf{A}) and mixed action/fluent observations (\textbf{A+F}). Figure \ref{fig:G_Comparison} compares the size of $\G^*_{cpx}$ and $\G^*_{ign}$ for different settings. These only consider instances where $|\G^*_{ign}| > 1$, leaving room for improvement. We report the number of these instances as n. Figure \ref{fig:TimeComparison} looks at the difference in time to compute $\G^*_{ign}$ vs the time to compute $\G^*_{cpx}$. This is divided by settings, and result of computation (if $\G^*_{ign}$ already optimal, if $\G^*_{ign}$ not optimal, but $\G^*_{cpx}$ did not improve, and if $\G^*_{cpx}$ did improve). The n for each category is reported. Note that values are always negative, meaning our method was always slower for the Block-Words domain.

\begin{table*}[tb]
    \centering
    \caption{Empirical Evaluation Results Per Domain and Setting}
    \resizebox{\textwidth}{!}{%
        \begin{tabular}{c|c|c|c||c|c|c|c|c|c|c|c|c|c|}
            && \textbf{U\%}
            & \textbf{D\%}
            & Opt
            & Imp 
            & $|\Theta_{ign}|$ Opt 
            & $|\Theta_{ign}|$ Imp 
            & $|\Theta_{cpx}|$ Opt
            & $|\Theta_{cpx}|$ Imp 
            & $|\G^*_{ign}|$ Imp
            & $|\G^*_{cpx}|$ Imp
            & time$_{ign}$ All 
            & time$_{cpx}$ All \\ \hline 
            \parbox[t]{2mm}{\multirow{20}{*}{\rotatebox[origin=c]{90}{\textbf{A}: Action Observations Only}}}& 
            \parbox[t]{2mm}{\multirow{5}{*}{\rotatebox[origin=c]{90}{Block-Words}}} 
              & 0\%  & 0\%  & 97 & 86  & 4.45 $\pm$ 0.24 & 4.08 $\pm$ 0.27 & 4.45 $\pm$ 0.24 & 4.08 $\pm$ 0.27 & \textbf{3.03 $\pm$ 0.33} & 3.03 $\pm$ 0.33 & 59.98 $\pm$ 3.25 & 77.11 $\pm$ 3.82  \\ 
            & & 0\%  & 25\% & 57 & 126 & 3.18 $\pm$ 0.19 & 2.75 $\pm$ 0.16 & 4.67 $\pm$ 0.29 & 4.10 $\pm$ 0.22 & 4.27 $\pm$ 0.48 & 3.17 $\pm$ 0.36 & 46.57 $\pm$ 3.08 & 93.38 $\pm$ 5.00  \\ 
            & & 25\% & 0\%  & 97 & 86  & 3.98 $\pm$ 0.15 & 3.78 $\pm$ 0.21 & 4.42 $\pm$ 0.23 & 4.12 $\pm$ 0.29 & 3.27 $\pm$ 0.33 & 3.02 $\pm$ 0.32 & 54.83 $\pm$ 2.97 & 83.55 $\pm$ 4.60  \\ 
            & & 50\% & 0\%  & 65 & 118 & 2.97 $\pm$ 0.14 & 2.81 $\pm$ 0.12 & 4.45 $\pm$ 0.28 & 4.19 $\pm$ 0.23 & 3.81 $\pm$ 0.40 & \textbf{2.86 $\pm$ 0.33} & 47.92 $\pm$ 2.98 & 99.81 $\pm$ 5.98  \\ 
            & & 50\% & 25\% & 46 & 137 & 2.61 $\pm$ 0.23 & 2.14 $\pm$ 0.13 & 4.59 $\pm$ 0.33 & 4.18 $\pm$ 0.21 & 5.01 $\pm$ 0.60 & 3.42 $\pm$ 0.45 & 41.12 $\pm$ 3.04 & 115.78 $\pm$ 6.75  \\  \cline{2-14}
            &\parbox[t]{2mm}{\multirow{5}{*}{\rotatebox[origin=c]{90}{Ipc-Grid}}} 
             & 0\%  & 0\%  & 76 & 14 & 6.92 $\pm$ 0.53 & 7.21 $\pm$ 1.24 & 6.92 $\pm$ 0.53 & 7.21 $\pm$ 1.24 & 2.00 $\pm$ 0.00 & 2.00 $\pm$ 0.00 & 4.20 $\pm$ 0.84 & 8.53 $\pm$ 1.75  \\ 
            && 0\%  & 25\% & 76 & 14 & 4.82 $\pm$ 0.41 & 5.14 $\pm$ 0.87 & 6.89 $\pm$ 0.53 & 7.36 $\pm$ 1.19 & 2.29 $\pm$ 0.42 & \textbf{1.93 $\pm$ 0.42} & 3.62 $\pm$ 0.74 & 8.01 $\pm$ 1.62  \\ 
            && 25\% & 0\%  & 74 & 16 & 5.78 $\pm$ 0.41 & 6.25 $\pm$ 0.79 & 6.84 $\pm$ 0.54 & 7.56 $\pm$ 1.03 & \textbf{2.12 $\pm$ 0.27} & 2.06 $\pm$ 0.31 & 3.63 $\pm$ 0.72 & 8.41 $\pm$ 1.74  \\ 
            && 50\% & 0\%  & 73 & 17 & 4.38 $\pm$ 0.35 & 4.65 $\pm$ 0.85 & 6.89 $\pm$ 0.54 & 7.29 $\pm$ 1.13 & 2.41 $\pm$ 0.66 & 1.94 $\pm$ 0.34 & 3.22 $\pm$ 0.65 & 10.88 $\pm$ 2.58  \\ 
            && 50\% & 25\% & 65 & 25 & 3.42 $\pm$ 0.34 & 3.16 $\pm$ 0.60 & 6.91 $\pm$ 0.59 & 7.12 $\pm$ 0.85 & 2.40 $\pm$ 0.55 & 1.64 $\pm$ 0.29 & 2.79 $\pm$ 0.54 & 11.70 $\pm$ 2.78  \\ 
            \cline{2-14}
            &\parbox[t]{2mm}{\multirow{5}{*}{\rotatebox[origin=c]{90}{Navigation}}} 
             & 0\%  & 0\%  & 58 & 5  & 9.31 $\pm$ 1.42 & 5.40 $\pm$ 0.68 & 9.31 $\pm$ 1.42 & 5.40 $\pm$ 0.68 & 3.60 $\pm$ 2.72 & 3.60 $\pm$ 2.72 & 0.20 $\pm$ 0.04 & 0.21 $\pm$ 0.02  \\ 
            && 0\%  & 25\% & 52 & 11 & 6.17 $\pm$ 1.13 & 6.55 $\pm$ 2.56 & 8.90 $\pm$ 1.50 & 9.45 $\pm$ 3.36 & 2.73 $\pm$ 0.85 & 2.09 $\pm$ 0.63 & 0.21 $\pm$ 0.07 & 0.19 $\pm$ 0.03  \\ 
            && 25\% & 0\%  & 56 & 7  & 7.36 $\pm$ 1.11 & 7.57 $\pm$ 5.36 & 8.96 $\pm$ 1.37 & 9.29 $\pm$ 6.51 & 2.71 $\pm$ 1.38 & 2.57 $\pm$ 1.50 & 0.19 $\pm$ 0.05 & 0.17 $\pm$ 0.02  \\ 
            && 50\% & 0\%  & 56 & 7  & 5.80 $\pm$ 0.91 & 6.29 $\pm$ 4.33 & 8.91 $\pm$ 1.37 & 9.71 $\pm$ 6.31 & \textbf{2.43 $\pm$ 0.73} & 2.00 $\pm$ 0.53 & 0.19 $\pm$ 0.03 & 0.19 $\pm$ 0.02  \\ 
            && 50\% & 25\% & 52 & 11 & 4.58 $\pm$ 0.83 & 4.64 $\pm$ 2.00 & 8.94 $\pm$ 1.44 & 9.27 $\pm$ 4.08 & 3.00 $\pm$ 1.08 & \textbf{1.91 $\pm$ 0.97} & 0.19 $\pm$ 0.04 & 0.20 $\pm$ 0.03  \\ 
            \cline{2-14}
            &\parbox[t]{2mm}{\multirow{5}{*}{\rotatebox[origin=c]{90}{Logistics}}} 
             & 0\%  & 0\%  & 54 & 6  & 9.83 $\pm$ 0.10 & 10.00 $\pm$ 0.00 & 9.83 $\pm$ 0.10 & 10.00 $\pm$ 0.00 & \textbf{2.00 $\pm$ 0.00} & 2.00 $\pm$ 0.00 & 892.68 $\pm$ 22.73 & 903.62 $\pm$ 22.72  \\ 
            && 0\%  & 25\% & 52 & 8  & 6.83 $\pm$ 0.11 & 7.00 $\pm$ 0.00 & 9.83 $\pm$ 0.11 & 10.00 $\pm$ 0.00 & 2.12 $\pm$ 0.30 & 2.00 $\pm$ 0.45 & 902.31 $\pm$ 22.47 & 900.95 $\pm$ 22.82  \\ 
            && 25\% & 0\%  & 45 & 15 & 7.84 $\pm$ 0.11 & 7.87 $\pm$ 0.19 & 9.84 $\pm$ 0.11 & 9.87 $\pm$ 0.19 & 2.13 $\pm$ 0.19 & 1.73 $\pm$ 0.33 & 884.55 $\pm$ 24.70 & 903.10 $\pm$ 24.98  \\ 
            && 50\% & 0\%  & 49 & 11 & 6.86 $\pm$ 0.10 & 6.82 $\pm$ 0.27 & 9.86 $\pm$ 0.10 & 9.82 $\pm$ 0.27 & 2.18 $\pm$ 0.27 & 1.55 $\pm$ 0.35 & 893.55 $\pm$ 24.43 & 917.98 $\pm$ 22.67  \\ 
            && 50\% & 25\% & 47 & 13 & 5.28 $\pm$ 0.24 & 5.08 $\pm$ 0.39 & 9.83 $\pm$ 0.11 & 9.92 $\pm$ 0.17 & 2.46 $\pm$ 0.58 & \textbf{1.31 $\pm$ 0.29} & 888.69 $\pm$ 27.90 & 923.47 $\pm$ 20.64  \\ 
            \hline\hline
            \parbox[t]{2mm}{\multirow{20}{*}{\rotatebox[origin=c]{90}{\textbf{A+F} : Action and Fluent Observations}}}& 
            \parbox[t]{2mm}{\multirow{5}{*}{\rotatebox[origin=c]{90}{Block-Words}}} 
             & 0\%  & 0\%  & 102 & 81  & 4.65 $\pm$ 0.25 & 4.51 $\pm$ 0.41 & 8.69 $\pm$ 0.43 & 8.40 $\pm$ 0.61 & \textbf{3.41 $\pm$ 0.53} & 2.57 $\pm$ 0.34 & 61.95 $\pm$ 3.71 & 112.73 $\pm$ 4.74  \\ 
            && 0\%  & 25\% & 44  & 132 & 3.16 $\pm$ 0.41 & 1.98 $\pm$ 0.17 & 9.00 $\pm$ 0.71 & 8.45 $\pm$ 0.42 & 6.36 $\pm$ 0.80 & 2.36 $\pm$ 0.30 & 38.17 $\pm$ 3.33 & 144.43 $\pm$ 6.44  \\ 
            && 25\% & 0\%  & 99  & 84  & 4.44 $\pm$ 0.22 & 4.06 $\pm$ 0.30 & 8.73 $\pm$ 0.44 & 8.36 $\pm$ 0.59 & 3.50 $\pm$ 0.46 & 2.77 $\pm$ 0.34 & 59.14 $\pm$ 3.44 & 140.76 $\pm$ 6.74  \\ 
            && 50\% & 0\%  & 89  & 94  & 4.10 $\pm$ 0.24 & 3.27 $\pm$ 0.26 & 9.15 $\pm$ 0.45 & 8.00 $\pm$ 0.54 & 3.70 $\pm$ 0.47 & 2.49 $\pm$ 0.29 & 53.20 $\pm$ 3.09 & 165.64 $\pm$ 7.57  \\ 
            && 50\% & 25\% & 41  & 124 & 2.51 $\pm$ 0.29 & 2.02 $\pm$ 0.17 & 8.63 $\pm$ 0.59 & 8.81 $\pm$ 0.46 & 6.30 $\pm$ 0.75 & \textbf{2.35 $\pm$ 0.33} & 36.13 $\pm$ 3.03 & 190.01 $\pm$ 8.16  \\ 
            \cline{2-14}
            &\parbox[t]{2mm}{\multirow{5}{*}{\rotatebox[origin=c]{90}{Ipc-Grid}}} 
             & 0\%  & 0\%  & 79 & 11 & 6.76 $\pm$ 0.57 & 7.18 $\pm$ 1.34 & 13.25 $\pm$ 1.01 & 14.73 $\pm$ 2.44 & \textbf{2.00 $\pm$ 0.00} & 2.00 $\pm$ 0.00 & 4.16 $\pm$ 0.83 & 1.23 $\pm$ 0.24  \\ 
            && 0\%  & 25\% & 62 & 27 & 3.71 $\pm$ 0.44 & 2.67 $\pm$ 0.56 & 13.53 $\pm$ 1.19 & 13.33 $\pm$ 1.53 & 3.26 $\pm$ 0.97 & \textbf{1.37 $\pm$ 0.19} & 2.76 $\pm$ 0.58 & 1.60 $\pm$ 0.42  \\ 
            && 25\% & 0\%  & 73 & 17 & 6.03 $\pm$ 0.52 & 6.82 $\pm$ 1.18 & 12.86 $\pm$ 1.01 & 15.88 $\pm$ 2.15 & \textbf{2.00 $\pm$ 0.00} & 1.76 $\pm$ 0.22 & 3.98 $\pm$ 0.80 & 1.81 $\pm$ 0.40  \\ 
            && 50\% & 0\%  & 71 & 19 & 5.52 $\pm$ 0.50 & 5.84 $\pm$ 1.21 & 13.13 $\pm$ 1.06 & 14.58 $\pm$ 2.00 & 2.26 $\pm$ 0.39 & 1.79 $\pm$ 0.20 & 3.76 $\pm$ 0.75 & 2.34 $\pm$ 0.65  \\ 
            && 50\% & 25\% & 51 & 31 & 3.18 $\pm$ 0.38 & 3.06 $\pm$ 0.54 & 13.61 $\pm$ 1.27 & 14.52 $\pm$ 1.30 & 3.13 $\pm$ 0.82 & 1.58 $\pm$ 0.18 & 3.06 $\pm$ 0.64 & 2.59 $\pm$ 0.61  \\
            \cline{2-14}
            &\parbox[t]{2mm}{\multirow{5}{*}{\rotatebox[origin=c]{90}{Navigation}}} 
             & 0\%  & 0\%  & 54 & 9  & 9.17 $\pm$ 1.49 & 8.33 $\pm$ 4.25 & 17.41 $\pm$ 2.79 & 17.56 $\pm$ 9.35 & 2.56 $\pm$ 1.02 & 2.44 $\pm$ 1.09 & 0.18 $\pm$ 0.02 & 0.23 $\pm$ 0.03  \\ 
            && 0\%  & 25\% & 48 & 14 & 4.56 $\pm$ 0.95 & 4.07 $\pm$ 1.61 & 16.73 $\pm$ 2.91 & 20.14 $\pm$ 6.89 & 3.07 $\pm$ 0.86 & 1.79 $\pm$ 0.79 & 0.18 $\pm$ 0.02 & 0.20 $\pm$ 0.04  \\ 
            && 25\% & 0\%  & 56 & 7  & 8.02 $\pm$ 1.31 & 9.29 $\pm$ 5.92 & 17.27 $\pm$ 2.71 & 18.71 $\pm$ 12.51 & 3.14 $\pm$ 1.35 & 2.57 $\pm$ 1.68 & 0.17 $\pm$ 0.01 & 0.20 $\pm$ 0.03  \\ 
            && 50\% & 0\%  & 57 & 6  & 7.60 $\pm$ 1.14 & 8.17 $\pm$ 6.28 & 17.16 $\pm$ 2.67 & 20.00 $\pm$ 15.12 & \textbf{2.50 $\pm$ 0.88} & 1.67 $\pm$ 0.54 & 0.17 $\pm$ 0.01 & 0.20 $\pm$ 0.03  \\ 
            && 50\% & 25\% & 39 & 21 & 3.95 $\pm$ 0.81 & 4.52 $\pm$ 1.12 & 16.21 $\pm$ 3.14 & 20.62 $\pm$ 5.36 & 2.86 $\pm$ 0.52 & \textbf{1.19 $\pm$ 0.18} & 0.16 $\pm$ 0.01 & 0.19 $\pm$ 0.01  \\
            \cline{2-14}
            &\parbox[t]{2mm}{\multirow{5}{*}{\rotatebox[origin=c]{90}{Logistics}}} 
            
             & 0\%  & 0\%  & 55 & 5  & 10.13 $\pm$ 0.43 & 19.25 $\pm$ 0.20 & 8.60 $\pm$ 0.68 & 19.80 $\pm$ 0.56 & \textbf{2.00 $\pm$ 0.00} & 1.80 $\pm$ 0.56 & 895.50 $\pm$ 22.44 & 913.19 $\pm$ 21.46  \\
            && 0\%  & 25\% & 35 & 25 & 5.71 $\pm$ 0.54 & 19.23 $\pm$ 0.24 & 4.68 $\pm$ 0.70 & 19.40 $\pm$ 0.32 & 2.60 $\pm$ 0.46 & 1.32 $\pm$ 0.26 & 862.35 $\pm$ 27.86 & 901.43 $\pm$ 26.05  \\ 
            && 25\% & 0\%  & 52 & 8  & 9.10 $\pm$ 0.41 & 19.25 $\pm$ 0.20 & 10.25 $\pm$ 1.72 & 19.62 $\pm$ 0.62 & \textbf{2.00 $\pm$ 0.00} & 1.75 $\pm$ 0.39 & 891.18 $\pm$ 22.76 & 910.72 $\pm$ 22.88  \\ 
            && 50\% & 0\%  & 47 & 13 & 7.94 $\pm$ 0.37 & 19.26 $\pm$ 0.22 & 8.00 $\pm$ 0.55 & 19.46 $\pm$ 0.31 & 2.08 $\pm$ 0.17 & 1.54 $\pm$ 0.31 & 890.45 $\pm$ 23.33 & 933.17 $\pm$ 23.85  \\ 
            && 50\% & 25\% & 37 & 23 & 4.89 $\pm$ 0.37 & 19.30 $\pm$ 0.23 & 4.26 $\pm$ 0.68 & 19.30 $\pm$ 0.33 & 2.48 $\pm$ 0.45 & \textbf{1.22 $\pm$ 0.29} & 866.39 $\pm$ 32.33 & 930.81 $\pm$ 22.57  \\ 
            \hline
            \end{tabular}%
    }%
    \caption{\textbf{U\%} is percent of observations placed in an unordered set. \textbf{D\%} is percent of `debound' observations. We distinguish between samples perfectly solved by the ignore strategy (Opt), and samples with room for improvement (Imp). $|\Theta_{ign/cpx}|$ Opt/Imp is the observation set size for the specified method and sample group. $|\G^*_{ign/cpx}|$ is the size of the solution set, per method, over the improvable (Imp) samples. time$_{ign/cpx}$ is the time to compute, per method, over all samples. 
    }
    \label{tab:Data}
\end{table*}

\begin{figure}
    \centering
    \includegraphics[width=\linewidth]{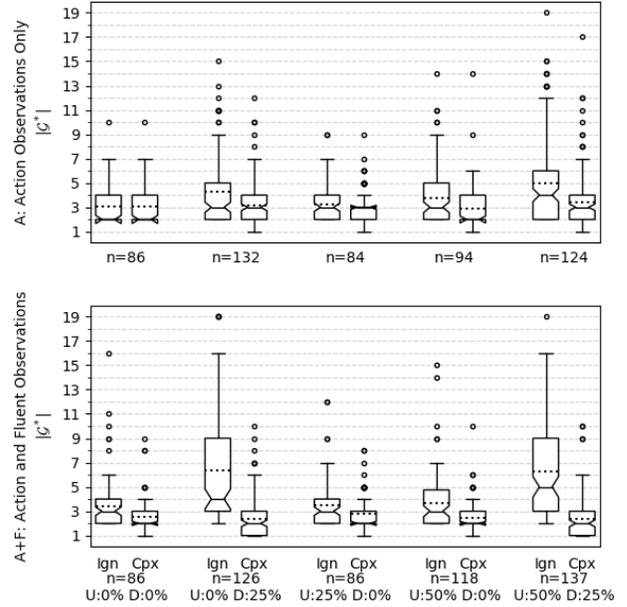}
    \caption{Comparison of solution set sizes $|G^*_{ign}|$ and $|G^*_{cpx}|$, from samples where improvement was possible. A solution set size of 1 is optimal.}
    \label{fig:G_Comparison}
\end{figure}
\begin{figure}
    \centering
    \includegraphics[width=\linewidth]{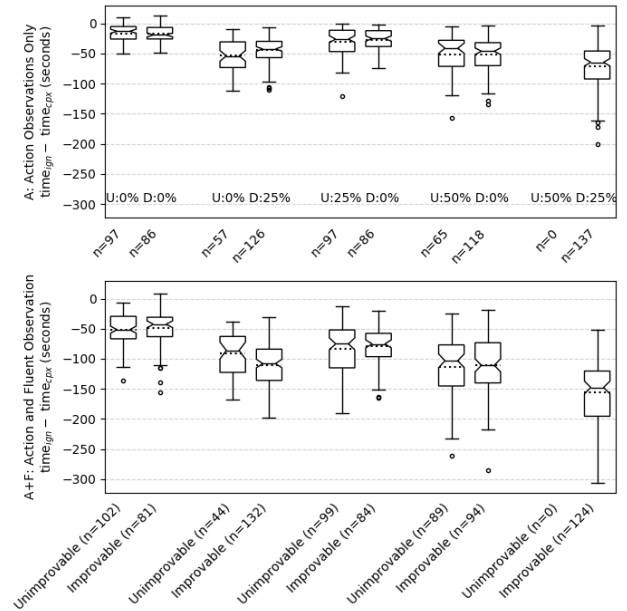}
    \caption{Difference in time to compute \citeauthor{ramirezGeffner09} method vs. time to compute this work's method.}
    \label{fig:TimeComparison}
\end{figure}

\subsection{Discussion}
Figure \ref{fig:G_Comparison} shows that complex observations can be a crucial factor in eliminating false hypotheses. Particularly for scenarios with multiple types of complexity, such as the (A+F U:50\% D:25\%) setting, ignoring complexity can cost three or four false positives. In no case were we less accurate, empirically confirming Theorem \ref{no_worse}.

This considered, our method is consistently slower across domains, regardless of improvement. We hypothesize that this is due to a larger search space. Utilizing more observations means including more actions in the planning domain, which might take longer to consider. This time is highly domain-dependent. For instance, Logistics takes hundreds of seconds while Ipc-Grid takes under a second.

For all domains except Block-Words, the number of instances where we could improve (\textit{i.e.} $|G^*_{ign}| \neq 1$) was too small to make significant conclusions. This brings up the concept of plan recognition difficulty. What makes Block-Words more difficult than the other domains? The other domains have, on average, larger observation sets to work with, derived from longer plans. Is it the number of observations available? Are the possible goals in its $\G$ more similar? If so, what makes them similar? Plan Recognition difficulty is not necessarily tied to planning difficulty. The Logistics domain took extraordinarily long compared to the Ipc-Grid and Navigation domains, yet all found the optimal solution set most of the time. 

In applications with plentiful information or a low ratio of complex to non-complex observations, ignoring complexity may be preferred for faster results with little loss of answer quality. However, areas with sparse information, higher ratios of complex to non-complex observations, or in domains known to be difficult, using complex information is vital, even if it takes longer to compute.

%
The new definitions for complex observation types can be used for any plan recognition approach, and our compilation can be adapted for other planning-based approaches. In particular, we are interested in adapting this compilation for probabilistic plan recognition and multi-agent plan recognition.

This work was limited by time constraints for how comprehensive an evaluation to run. We selected representative settings, but wish to reevaluate with more coverage over more settings to pinpoint those settings where a domain becomes `easy', as measured by how often the optimal solution set is found. 
%
%
%
\section{Conclusion}
For plan recognition to be used broadly, it needs to be capable of handling all types of information handed to it. From obstructed vision in robots to ambiguous word meanings in natural language, complex observations can come from many real-world scenarios, and this method lays the groundwork for leveraging them. We provide crisp definitions for partial-order optional observations of eaither fluents or actions and what it means to satisfy each, then prove that our compilation will produce satisfactory plans. While this work deals only with optimal solutions, the definitions provided can be extended to work with non-optimal probabilistic plan recognition. 

\bibliographystyle{aaai}
\bibliography{references}

\end{document}